\documentclass[letterpaper]{article}
\usepackage{proceed2e}
\usepackage[margin=1in]{geometry}

\usepackage{times}

\usepackage[all]{xy}
\usepackage{tabularx}
\usepackage{booktabs}
\usepackage{multirow}
\usepackage{graphicx}
\usepackage{mathnotation}
\usepackage{amsmath}
\usepackage{amssymb}
\usepackage{amsthm}
\usepackage{amsfonts}
\usepackage{relsize}
\usepackage{url}
\usepackage{enumerate}
\usepackage{xcolor}
\usepackage{color}

\usepackage{tikz}
\usepackage{etoolbox}
\newcommand{\circled}[2][]{%
	\tikz[baseline=(char.base)]{%
		\node[shape = circle, draw, inner sep = 0pt]
		(char) {\phantom{\ifblank{#1}{#2}{#1}}};%
		\node at (char.center) {\makebox[0pt][c]{#2}};}}
\robustify{\circled}

\usepackage{wrapfig}
\usepackage{framed}

\title{On the Discrimination Power and Effective Utilization \\ of Active Learning Measures in Version Space Search
}

\author{Patrick Rodler\\
Alpen-Adria Universit\"at Klagenfurt \\
patrick.rodler@aau.at
} 

\newcommand{\LD}{\mathcal{L}}
\newcommand{\UD}{\mathcal{U}}
\newcommand{\Hyp}{\mathcal{H}}
\newcommand{\HypP}{\mathcal{H}^+}
\newcommand{\HypN}{\mathcal{H}^-}
\newcommand{\HypZ}{\mathcal{H}^0}
\newcommand{\VP}{\mathit{V}^+}
\newcommand{\VN}{\mathit{V}^-}
\newcommand{\VZ}{\mathit{V}^0}
\newcommand{\V}{\mathcal{V}}
\newcommand{\Pt}[2]{\mathfrak{P}_{#1}(#2)}
\newcommand{\Part}{\mathfrak{P}}
\newcommand{\DPO}{\mathsf{DPO}}

\newtheorem{definition}{Definition}{}
{}
\newtheorem{proposition}{Proposition}{}
{}
{}

\begin{document}

\maketitle

\begin{abstract}
Active Learning (AL) methods have proven cost-saving against passive supervised methods in many application domains. An active learner, aiming to find some target hypothesis, formulates sequential queries to some oracle. The set of hypotheses consistent with the already answered queries is called version space. Several query selection measures (QSMs) for determining the best query to ask next have been proposed. 
Assuming binary-outcome queries, 
we analyze various QSMs wrt.\ to the discrimination power of their selected queries within the current version space.
As a result, we derive superiority
and equivalence relations between these QSMs and introduce improved versions of existing QSMs to overcome identified issues. 
The obtained picture gives a hint about which QSMs should preferably be used in pool-based AL scenarios.
Moreover, we deduce properties optimal queries wrt.\ QSMs must satisfy. Based on these, we demonstrate how efficient heuristic search methods for optimal queries in query synthesis AL scenarios can be devised.
%
\end{abstract}

\section{INTRODUCTION}\label{sec:intro}

	A supervised learning scenario where the 
learner
is allowed to choose the training data from which it learns is referred to as \emph{Active Learning (AL)} 
\cite{settles09}.
In AL, an oracle, e.g.\ a human expert, can be queried to label any 
query from a predefined query space. Given a set of labeled queries, the set of all hypotheses explaining the query labels is called \emph{version space} \cite{mitchell1977version}.
We assume that the learner maintains (a subset of) the current version space and uses sequential queries to the oracle to gradually refine it. To this end, we assume some update operator that takes a set of hypotheses and a new labeled query as input and returns a new set of hypotheses (possibly including previously unseen ones). So, the learner performs a \emph{version space search} for the best hypothesis.
%

AL has been successfully applied to a variety of domains such as text classification \cite{tong2001text}, image retrieval \cite{tong2001image}, concept learning \cite{DBLP:journals/ml/CohnAL94}, music retrieval \cite{mandel2006music}, machine translation \cite{ambati2010translation}, 
cancer classification \cite{liu2004cancer}, medical image classification \cite{hoi2006image}, reinforcement learning \cite{lopes2009reinforcement} and natural language processing \cite{olsson2009literature}.  
It often achieves significant (even exponential \cite{DBLP:conf/nips/Dasgupta04}) savings compared to ``passive'' supervised learning in terms of \emph{sample complexity} \cite{kulkarni1993active,balcan2010true}, i.e.\ querying cost. 
Hence, AL is especially useful when query labeling comes at high cost and there is a large amount of 
unlabeled data to choose from 
\cite{settles09}.

In the mentioned classical AL application scenarios (1)~hypotheses, e.g.\ decision trees or neural networks, usually give \emph{complete information} in that they predict a label for each query. In addition, (2)~unlabeled data is often 
\emph{cheaply obtainable}.
However, there are other use cases of AL where (1) and (2) do not hold. Such use cases can be found e.g.\ in hypotheses discrimination tasks arising in logic-based abduction \cite{kakas1992abductive}, 
theory selection \cite{sattar1991}, answer set programming 
\cite{brewka2011answer}, 
model-based diagnosis \cite{Reiter87,dekleer1987}, knowledge base debugging \cite{Rodler2015phd}, Semantic Web applications \cite{Kalyanpur.Just.ISWC07,Shchekotykhin2012} and ontology alignment repair \cite{meilicke2011}.
%
E.g., in a model-based diagnosis task one might ask 
``which faulty components in my car cause it not to start?''. The goal is then -- based on a (possibly incomplete) logical model describing the car -- to determine the actual explanation (actually faulty components) among a number of competing explanations for the faulty behavior of the car. In this scenario, an explanation $e$ (along with the model) might not predict any outcome for a specific test $t$ with the car \cite{dekleer1987}. So, no result of $t$ can rule out $e$.  
Also, (useful) unlabeled data, i.e.\ queries, might be costly to construct, e.g.\ when its computation relies on logical derivations from the given model 
\cite{dekleer1987,DBLP:journals/jair/FeldmanPG10a,Shchekotykhin2012,Rodler2015phd}. 
%
%

In any AL task the goal of a query is to discriminate well between competing hypotheses. 
At this, a minimal requirement usually postulated 
is that any query outcome must lead to the dismissal of at least some (known) hypothesis, i.e.\ it makes no sense to query something all known hypotheses agree about. This helps to initially restrict the query search space to the so-called \emph{region of uncertainty} \cite{DBLP:journals/ml/CohnAL94}. To extract an informative query from this region AL methods employ various \emph{query selection measures (QSMs)}, real-valued functions quantifying the quality of queries.
QSMs can be used in different AL scenarios such as \emph{query synthesis} and \emph{pool-based selection} \cite{settles09}. In the former a learner tries to \emph{generate} an unlabeled query with sufficiently good QSM-value. 
In the latter, the best query wrt.\ a QSM in a (usually large) set of unlabeled queries
is determined by comparing the QSM-value of all queries in the set.

The non-fulfillment of (1) might cause QSMs in a pool-based scenario to select queries with suboptimal discrimination power despite the presence of better queries in the pool. This issue is crucial when opting for a suitable QSM to be used in pool-based AL. 
%
%
The violation of (2), on the other hand, motivates the need for advanced query synthesis methods granting high query quality even though \emph{actually generating} only a small number of query candidates. 

\noindent\textbf{Contributions.} 
In this paper we analyze various AL QSMs 
and
\textbf{(1)}~define a plausible general discrimination preference order (DPO) on 
queries (formalizing the notion of ``discrimination power''), 
\textbf{(2)}~formally characterize a superiority relation
on QSMs based on the (degree of) their compliance with the DPO,
\textbf{(3)}~figure out superiority relationships between QSMs which suggests a preference order on QSMs helping to opt for the most suitable QSM in pool-based AL,
\textbf{(4)}~derive improved (parameterized) versions from some QSMs to overcome unveiled deficits,
\textbf{(5)}~formalize the notion of equivalence between QSMs based on their preference order on queries, 
\textbf{(6)}~give equivalence classes of QSMs under various conditions (query spaces, QSM parametrizations),
\textbf{(7)}~analyze QSM functions regarding their global optima and determine properties of optimal input arguments, and
\textbf{(8)}~show how these properties can be used to design heuristic search procedures for the \emph{systematic} construction of (nearly) optimal queries wrt.\ a QSM in an AL query synthesis scenario.

\section{PRELIMINARIES}\label{sec:basics}
In an AL setting we consider 
there is a set of unlabeled queries $\UD$ and a (possibly empty) set of already
labeled queries $\LD$. Each (unlabeled) \emph{query} is a sentence in first-order logic. A \emph{labeled query} in $\LD$ is a tuple $(Q,a_Q)$ where $Q$ is a query and $a_Q \in \setof{0,1}$. 
The answers\footnote{Note, we use \emph{answer} and \emph{label} interchangeably.} $a_Q = 1$ and $a_Q = 0$ mean that the first-order sentence represented by $Q$ is true and false, respectively. 
%
Queries are answered by an \emph{oracle} given by the total function $ans: \UD \to \setof{0,1}$ which maps queries $Q \in \UD$ to their respective answer $a_Q$.


The active learner attempts to find (an approximation of) the \emph{target hypothesis} $h_t$ from a hypothesis space $\Hyp$ which depends on the learning task. E.g., for a decision tree learning task each $h \in \Hyp$ is a candidate decision tree; for a model-based car diagnosis task each $h \in \Hyp$ is a possible diagnosis, i.e.\ an assumption about the faulty/healthy-state of each (relevant) component of the car that explains all observations about the car's (faulty) behavior. 

Due to the generality of the query notion, e.g.\ instance labels for binary ($\{0,1\}$) classification can be obtained by asking, say, $Q:=(h_t(i)=1)$ for some instance $i$, 
all concept learning query types discussed in \cite{DBLP:journals/ml/Angluin87} can be captured as well as observations or (system) tests in discrimination tasks mentioned in Sec.~\ref{sec:intro} can be specified as e.g.\ $Q:= (\mathit{carBatteryFlat})$. Moreover, classification (regression) model learners restricted to queries $Q:=(h_t(i) \in \mathit{rng})$ for a discrete (continuous) range $\mathit{rng}$ can be modeled.

Given a set of labeled queries $\LD$, any hypothesis $h \in \Hyp$ is still possible if it is \emph{consistent} with $\LD$. The set including all $h \in \Hyp$ consistent with $\LD$ is called the current \emph{version space} $\V \subseteq \Hyp$ \cite{DBLP:journals/ai/Mitchell82}. 
For tasks where each $h \in \Hyp$ gives \emph{complete} information (e.g.\ decision trees), each $h$ predicts 
a label for each query. In particular, all $h \in \V$ \emph{entail} all answers $a_Q$ for $(Q,a_Q) \in \LD$.
So, queries in this case make a binary discrimination between the competing hypotheses in $\Hyp$, cf.\ \cite{DBLP:journals/etai/BryantMOKRK01}.
However, in the more general setting we consider, hypotheses in $\Hyp$ might include \emph{incomplete} (e.g.\ logical) knowledge, and thus not entail any label for a query $Q$. 
In this case, all $h \in \V$ 
just \emph{not contradict} any $a_Q$ for $(Q,a_Q) \in \LD$.
%
Hence, in the general scenario each query $Q$ imposes a partition on $\Hyp$ 
into three sets $\langle\HypP_Q,\HypN_Q,\HypZ_Q\rangle$: 
$\HypP_Q$ includes those $h \in \Hyp$ consistent only with $a_Q = 1$ (predicting $Q$'s positive answer), $\HypN_Q$ those $h \in \Hyp$ consistent only with $a_Q = 0$ (predicting $Q$'s negative answer), and $\HypZ_Q$ those consistent with both $a_Q = 1$ and $a_Q = 0$ (not predicting any answer). 
That is, the new (still consistent) hypotheses set after $ans(Q) = 1$ is known (i.e.\ $(Q,1)$ is added to $\LD$) is $\Hyp\setminus\HypN_Q$. Otherwise, if 
$(Q,0)$ is added to $\LD$, the new hypotheses set is $\Hyp\setminus\HypP_Q$.

If the target hypothesis $h_t$ is in $\HypP_Q$ ($\HypN_Q$), then $ans(Q) = 1$ ($ans(Q) = 0$). We stress that the oracle is a \emph{total} function and thus assumed to answer \emph{every} query $Q \in \UD$, even if $h_t \in \HypZ_Q$. 
E.g., even if $h_t$ in a car diagnosis task 
neither entails $(lightWorks)$ nor $\lnot(lightWorks)$, an oracle (e.g.\ a car mechanic) can verify whether the light of the car works.
For either outcome, $h_t$ remains valid a-posteriori. 

As the explicit computation of the full version space $\V \subseteq \Hyp$ might be hard or even infeasible \cite{dekleer1987,DBLP:conf/colt/SeungOS92,DBLP:conf/isaim/DasguptaHM08,Shchekotykhin2012,Rodler2015phd}, we assume that some subset $V$ of $\V$ is known to the learner at each query selection. $V$ may comprise e.g.\ (some of) the most ``succinct'' \cite{DBLP:journals/jair/FeldmanPG10a}, most probable \cite{DBLP:conf/aaai/Kleer91}, most specific or most general \cite{mitchell1977version} hypotheses. 
As with $\Hyp$, a query $Q$ partitions $V$ into $\VP_Q := V \cap \HypP_Q$, $\VN_Q := V \cap \HypN_Q$ and $\VZ_Q := V \cap \HypZ_Q$. 
We denote by $\Pt{V}{Q} := \langle \VP_Q,\VN_Q,\VZ_Q\rangle$ the (unique) \emph{partition of $Q$ (wrt.\ $V$)}. Generally, multiple queries $Q$ might have the same partition  $\Pt{V}{Q}$.
We call $Q \in \UD$ a \emph{discriminating query (DQ) (wrt.\ $V$)} iff $V^+_Q \neq \emptyset$ and $V^-_Q \neq \emptyset$. 
Similarly, we call $\Pt{V}{Q}$ for a DQ $Q$ a \emph{discriminating partition (DP) (wrt.\ $V$)}.
That is, either label $a_Q \in \setof{0,1}$ of a DQ $Q$ eliminates at least one $h \in V$ or, respectively, at least two hypotheses in $V$ make different predictions as to $a_Q$. Intuitively, a learner will try to avoid to ask any $Q\in \UD$ which is not a DQ. Because -- based on the current evidence in terms of $V$ -- it cannot be sure that any relevant new information will be gained by obtaining $a_Q$. The DQs are exactly the elements of the \emph{region of uncertainty \cite{DBLP:journals/ml/CohnAL94} (wrt.\ $V$)}.
A DQ $Q$ is termed \emph{weak DQ (wrt.\ $V$)} iff $\VZ_Q \neq \emptyset$. Otherwise, we call $Q$ \emph{strong DQ (wrt.\ $V$)}.

%
An AL \emph{query selection measure (QSM)} is a function $m: \UD \to \mathbb{R}$ assigning to each query $Q \in \UD$ a (quality) measure $m(Q) \in \mathbb{R}$. A \emph{theoretical optimum $X$ wrt.\ $m$} is a hypothetical (not necessarily real) DQ $X$ 
which globally optimizes $m(X)$. 
Depending on the QSM $m$, ``optimizing $m$'' can mean either maximizing or minimizing $m$. An \emph{optimal query $Q$ wrt.\ $m$ and $V$} is a DQ wrt.\ $V$ with optimal $m(Q)$ among all DQs wrt.\ $V$. Note, theoretical optima and optimal queries need not be unique.

In line with the works \cite{dekleer1987,DBLP:journals/etai/BryantMOKRK01,Shchekotykhin2012,Rodler2015phd} we characterize a probability space over $\Hyp$ as follows: We assume that each $h \in \Hyp$ has an a-priori probability $p(h)$ of being the target hypothesis $h_t$, i.e.\ $p(h) := p(h = h_t)$. Given a currently known subset $V$ of the version space $\V \subseteq \Hyp$, we define 
$p(X) := \sum_{h \in X} p(h)$
for $X \subseteq V$ and assume $p$ to be normalized over $V$ such that that $p(V)=1$.
Since the version space includes only still possible hypotheses, $p(h) > 0$ must hold for all $h \in V$. 
For any $Q\in\UD$ and oracle $ans$: 
$p(ans(Q)=1) := p(\VP_Q)+\frac{p(\VZ_Q)}{2}$
and
$p(ans(Q)=0) = p(\VN_Q)+\frac{p(\VZ_Q)}{2}$
i.e.\ the non-predicting hypotheses $h \in \VZ_Q$ are assumed to predict each answer with a probability of $\frac{1}{2}$.
%
%
The posterior probability $p(h \mid ans(Q) = a_Q)$ of some $h \in \Hyp$ can be computed by the Bayesian Theorem as 
$p( ans(Q) = a_Q | h)\;\,p(h)/p(ans(Q) = a_Q)$
where 
$p(ans(Q)=1\mid h)$ is $1$ if $h \in \HypP_Q$, $0$ if $h \in \HypN_Q$, and $\frac{1}{2}$ if $h \in \HypZ_Q$. 
%
%

\renewcommand{\arraystretch}{1.1}
\begin{table}
	\caption{Running Example: Some sample partitions wrt.\ $V = \setof{h_1,\dots,h_5}$ (top) and 
		probability distributions $p_1$, $p_2$ and $p_3$ over $V$ (bottom).}
	\label{tab:example_QPs}
	\scriptsize
	\begin{center}
	\begin{tabular}{cccc}
		\rule[-6pt]{0pt}{16pt} $i$	& $\VP_{Q_i}$ & $\VN_{Q_i}$ & $\VZ_{Q_i}$ \\ 
		\hline 
		$1$	& $\setof{h_1,h_2}$ & $\setof{h_3,h_4,h_5}$ & $\emptyset$ \\ 
		$2$	& $\setof{h_1,h_2}$ & $\setof{h_3,h_4}$ & $\setof{h_5}$ \\ 
		$3$	& $\setof{h_4}$ & $\setof{h_1,h_2,h_3,h_5}$ & $\emptyset$ \\ 
		$4$	& $\setof{h_1,h_2,h_5}$ & $\setof{h_4}$ & $\setof{h_3}$ \\ 
	\end{tabular} 
	%
	%
	\begin{tabular}{lccccc}
		\toprule[1pt]  
		& $h_1$ & $h_2$ & $h_3$ & $h_4$ & $h_5$ \\ 
		\hline 
		$p_1(h_i)$	& $0.35$ & $0.05$ & $0.15$ & $0.25$ & $0.2$ \\
		$p_2(h_i)$	& $0.01$ & $0.02$ & $0.8$ & $0.15$ & $0.02$ \\
		$p_3(h_i)$	& $0.4$ & $0.2$ & $0.05$ & $0.1$ & $0.25$ \\
		\hline 
	\end{tabular}
	\end{center}
\end{table}


\textbf{Example:}
Consider Tab.~\ref{tab:example_QPs} which gives some partitions $\Pt{V}{Q_i}$ of $V := \setof{h_1,\dots,h_5}$ for $1\leq i \leq 4$. All associated queries $Q_i$ (not explicitly given in Tab.~\ref{tab:example_QPs}) are DQs as $\VP_{Q_i}$ and $\VN_{Q_i}$ are non-empty for $1\leq i \leq 4$. Hence, each partition in the table is a DP. Moreover, $Q_1,Q_3$ are strong and $Q_2,Q_4$ weak DQs due to empty and non-empty $\VZ_{Q_i}$, respectively.

Assuming the probabilities $p := p_1$ over $V$ (see Tab.~\ref{tab:example_QPs}), e.g.\ $p(ans(Q_3)=1) = p(\VP_{Q_3}) = p(\setof{h_4}) = 0.25$ and $p(ans(Q_2)=0) = p(\VN_{Q_2}) + \frac{1}{2} p(\VZ_{Q_2}) = p(\setof{h_3,h_4}) + \frac{1}{2} p(\setof{h_5}) = 0.15+0.25+\frac{1}{2} 0.2 = 0.5$.

Let $m_1(Q) := |p(\VP_Q) - p(\VN_Q)| + p(\VZ_Q)$ be a QSM (to be minimized). Then $\langle m_1(Q_1),\dots,m_1(Q_4)\rangle = \langle 0.2, 0.2, 0.5, 0.5\rangle$. Supposing that $Q_1, \dots, Q_4$ are \emph{all} possible DQs wrt.\ $V$, the optimal queries wrt.\ $m_1$ and $V$ are $Q_1$ and $Q_2$. A theoretical optimum $X$ wrt.\ $m_1$ satisfies $p(\VP_{X}) = p(\VN_{X}) = 0.5$ and $p(\VZ_{X}) = 0$. 

Suppose $Q_2$ is labeled negatively, i.e.\ $ans(Q_2) = 0$. Then the hypotheses $h_1,h_2$ are invalidated. The remaining ones are $V \setminus \VP_{Q_2} = \setof{h_3,h_4,h_5}$. The (Bayes) updated probability distribution over $V$ is then 
$p(h_1) = p(h_2) = 0$, $p(h_3) = \frac{0.15}{0.5} = 0.3$, $p(h_4) = \frac{0.25}{0.5} = 0.5$ and $p(h_5) = \frac{(1/2) 0.2}{0.5} = 0.2$.
\qed

\section{ANALYSIS OF QUERY SELECTION MEASURES}\label{sec:contrib}
In this section\footnote{\emph{Detailed} proofs of all results are given in \cite[Sec.~3.2 ff.]{zattach}.} we motivate and specify a general discrimination-preference order (DPO) over queries in $\UD$, study various QSMs regarding 
their compliance with the DPO, present derived equivalence and superiority relations among these QSMs and specify some plausible new QSMs, e.g.\ as improved versions of existing ones. The results facilitate the decision upon which QSM to use in \emph{pool-based AL scenarios}. Moreover, we analyze the QSM functions wrt.\ their (theoretically) optimal inputs which lets us deduce properties of optimal strong DQs for the discussed QSMs. These properties provide the basis for a systematic construction of (or search for) optimal DQs in a \emph{query synthesis AL scenario}.    

We first point out that the partition of a query $Q \in \UD$ (along with the probability measure $p$) gives already all the relevant information that is taken into account by QSMs used for version space search. Because the partition enables \\
\textbf{(1)}~the verification whether a query $Q$ is a DQ  (i.e.\ whether the query is in the region of uncertainty), \\
\textbf{(2)}~the test whether $Q$ is strong (i.e.\ $\VZ_Q =\emptyset$), \\
\textbf{(3)}~an estimation of the impact $Q$'s answers have in terms of hypotheses elimination (potential \mbox{a-posteriori} change of the version space), \\
\textbf{(4)}~the assessment of the probability of $Q$'s positive and negative answers (e.g.\ to determine the uncertainty of $Q$).
%
%
%
%
%
%

\noindent\textbf{Relevant Definitions and Properties.} QSMs
might basically focus on pretty different properties of a query's partition when estimating its goodness. However,
independently of the concrete used QSM, 
queries with a higher ``discrimination power'' should be preferred. Intuitively, given a query $Q_1 \in \UD$ which is objectively better than $Q_2 \in \UD$, we do not want a reasonable QSM to propose $Q_2$.
%
We next define a general order on queries, called DPO, thereby formalizing the notion of ``discrimination power''.
Note, in the following we always assume $\V$ to be the current version space and $V \subseteq \V$. 
\begin{definition}\label{def:DPO}
	Let $Q,\overline{Q} \in \UD$. 
	Further, for any query $Q\in \UD$ let $V_Q[\lnot a] \subseteq V$ denote the hypotheses predicting $\lnot a$ (i.e.\ inconsistent with $ans(Q)=a$). That is, exactly $V_Q[\lnot a]$ is eliminated among all hypotheses in $V$ given that $Q$ is answered by $a$. 
	
	Then we call $Q$ \emph{discrimination-preferred to $\overline{Q}$ (wrt.\ $V$)} iff there is an injective function $f:\setof{0,1}\to\setof{0,1}$ that maps each of $\overline{Q}$'s answers $\overline{a}_1, \overline{a}_2 \in \setof{0,1}$ ($\overline{a}_1 \neq \overline{a}_2$) 
	to one of $Q$'s answers $a_i = f(\overline{a}_i)$ such that \\
	%
	%
	\textbf{(1)}~\label{def:DPO:cond1} $V_Q[\lnot a_i] \supseteq V_{\overline{Q}}[\lnot \overline{a}_i]$ for some $i \in \setof{1,2}$, and \\
	\textbf{(2)}~\label{def:DPO:cond2} $V_Q[\lnot a_j] \supset V_{\overline{Q}}[\lnot \overline{a}_j]$ for $j \in \setof{1,2}$ and $j \neq i$. 
	
	We use $Q \prec_{\DPO} \overline{Q}$ to state that $Q$ is discrimination-preferred to $\overline{Q}$ and call $\{(Q,\overline{Q})\mid Q \prec_{\DPO} \overline{Q} \}$ \emph{the discrimination preference order (DPO)}.
\end{definition}
Simply put, $Q \prec_{\DPO} \overline{Q}$ means:
For each result one might get by asking the oracle $\overline{Q}$, there is a better result in terms of hypotheses elimination one can get by asking $Q$. In particular, for one of the answers $\overline{a}_i$ of $\overline{Q}$, some answer $a_i$ to $Q$ eliminates at least the same hypotheses. For the other answer $\overline{a}_j (\neq \overline{a}_i)$ of $\overline{Q}$, the other answer $a_j (\neq a_i)$ to $Q$ eliminates strictly more hypotheses.

%
%
%
The idea underlying the DPO is that asking $Q$ is \emph{always} (i.e.\ for any answer) better than asking $\overline{Q}$ given that the target hypothesis is in $V$ and predicts an answer for both queries:
\begin{proposition}\label{prop:if_ht_in_V_and_predicting_then_disc-pref_query_always_better}
	Let $Q \prec_{\DPO} \overline{Q}$ 
	and the target hypothesis $h_t \in \VP_{Q} \cup \VN_{Q}$ and $h_t \in \VP_{\overline{Q}} \cup \VN_{\overline{Q}}$. Then the remaining hypotheses in $V$ after adding $(Q,ans(Q))$ to $\LD$ is a subset of the remaining hypotheses in $V$ after adding $(\overline{Q},ans(\overline{Q}))$ to $\LD$.
\end{proposition}
\begin{proof} 
	The proposition follows from the fact that (i)~for any $Q \in \UD$, $ans(Q) = 1$ if $h_t \in \VP_Q$ and $ans(Q) = 0$ if $h_t \in \VN_Q$, that (ii)~$(h_t \in \VP_{Q}) \oplus (h_t \in \VN_{Q})$ and $(h_t \in \VP_{\overline{Q}}) \oplus (h_t \in \VN_{\overline{Q}})$, and (iii)~the subset-relations in 
	(1)
	and 
	(2)
	in Def.~\ref{def:DPO}.
\end{proof}
\textbf{Example (cont'd):}
In Tab.~\ref{tab:example_QPs}, $Q_1 \prec_{\mathsf{DPO}} Q_2$ and $Q_3 \prec_{\mathsf{DPO}} Q_4$. E.g.\ the latter, by Def.~\ref{def:DPO}, holds since (1)~for $ans(Q_3) = 0$, which eliminates $\setof{h_4}$, there is an answer, namely $ans(Q_4) = 1$, which also dismisses $\setof{h_4}$, and (2)~for $ans(Q_3) = 1$ (making $\setof{h_1,h_2,h_3,h_5}$ invalid) the answer $ans(Q_4) = 0$ is strictly worse (invalidating only $\setof{h_1,h_2,h_5}$).

Given e.g.\ $h_t \in \setof{h_1,h_2,h_4,h_5}$, 
then the hypothesis elimination rate (wrt.\ $V$) achieved by the discrimination-preferred $Q_3$ is  better than the one of $Q_4$ for any oracle $ans$ (Prop.~\ref{prop:if_ht_in_V_and_predicting_then_disc-pref_query_always_better}).
\qed

Every QSM imposes a (preference) order on a given set of queries $\UD$:
\begin{definition}\label{def:precedence_order_defined_by_q-partition_quality_measure} 
	Let $m$ be a QSM and $Q,Q' \in \UD$. Then \emph{$Q$ is preferred to $Q'$ by $m$}, formally $Q \prec_m Q'$, iff 
	%
	\textbf{(a)}~$m(Q) < m(Q')$ if $m$ is optimized by minimization,
	\textbf{(b)}~$m(Q) > m(Q')$ if $m$ is optimized by maximization.
\end{definition}
Two QSMs are equivalent iff they impose exactly the same preference order on queries: 
\begin{definition}\label{def:measures_equivalent}
	Let $m_1,m_2$ be QSMs.
	Then we call \emph{$m_1$ equivalent to $m_2$} (\emph{$m_1$ $\mathfrak{X}$-equivalent to $m_2$}), formally $m_1 \equiv m_2$ ($m_1\equiv_{\mathfrak{X}} m_2$), iff for all queries ${Q,Q' \in (\mathfrak{X} \subseteq)\; \UD}$: $Q \prec_{m_1} Q'$ iff $Q \prec_{m_2} Q'$.
\end{definition}
The next definition facilitates our analysis of the degree of compliance of QSMs with the DPO:
\begin{definition}\label{def:measure_satisfies_consistentWith_DPO}
	Let $m$ be a QSM. We say that $m$ 
	\emph{preserves (or: satisfies) the DPO (over $\mathfrak{X}$)} iff whenever $Q \prec_{\DPO} Q'$ (and $Q,Q' \in \mathfrak{X}$), it holds that $Q \prec_{m} Q'$ 
	(i.e.\ the preference order imposed on queries by $m$ is a superset of the DPO). \\
	Further, we call $m$ \emph{consistent with the DPO (over $\mathfrak{X}$)} iff whenever $Q \prec_{\DPO} Q'$ (and $Q,Q' \in \mathfrak{X}$), it does not hold that $Q' \prec_{m} Q$ (i.e.\ the preference order imposed on queries by $m$ has an empty intersection with the inverse DPO).
\end{definition} 
We call QSMs with a higher compliance with the DPO superior to others:
\begin{definition}\label{def:measure_superior}
	Let $m_1,m_2$ be QSMs. We call $m_2$ \emph{superior to} $m_1$ (or: $m_1$ \emph{inferior to }$m_2$), formally $m_2 \prec m_1$, iff \\
	\textbf{(1)}~for some pair of queries $Q,Q'$ where $Q \prec_{\DPO} Q'$ and not $Q \prec_{m_1} Q'$ it holds that $Q \prec_{m_2} Q'$ 
	(i.e.\ in some cases $m_2$ does, but $m_1$ does not satisfy the DPO),
	and \\
	\textbf{(2)} for no pair of queries $Q,Q'$ where $Q \prec_{\DPO} Q'$ and not $Q \prec_{m_2} Q'$ it holds that $Q \prec_{m_1} Q'$
	(i.e.\ whenever $m_2$ does not satisfy the DPO, $m_1$ does not satisfy it either). \\
	%
	%
	Analogously, we call $m_2$ \emph{$\mathfrak{X}$-superior to }$m_1$ (or: $m_1$ \emph{$\mathfrak{X}$-inferior to }$m_2$), formally $m_2 \prec_{\mathfrak{X}} m_1$, iff 
	superiority of $m_2$ to $m_1$ holds over $\mathfrak{X} \subseteq \UD$.	
\end{definition}

The following proposition can be easily verified:
\begin{proposition}\label{prop:precedence_order_is_strict_order}
	$\prec_m$ and $\prec_{\DPO}$ are strict orders, i.e.\ irreflexive, asymmetric and transitive relations over queries. 
	$\equiv$ and $\equiv_{\mathfrak{X}}$ are equivalence relations over QSMs.
	$\prec$ and $\prec_{\mathfrak{X}}$ are strict orders over QSMs.
\end{proposition}
The next proposition summarizes some easy consequences of the provided definitions:
\begin{proposition}\label{prop:properties_derived_from_definitions}
	Let $m,m_1,m_2$ be QSMs, $Q,Q' \in \UD$, $\mathfrak{X} \subseteq \UD$
	and $Q_{m_i} \in \UD$ denote the optimal query wrt.\ $m_i \; (i\in \setof{1,2})$ and $V$. 
	Then: \\
	\textbf{(1)} 
	$m_1 \equiv m_2$ implies $Q_{m_1} = Q_{m_2}$.\\
	\textbf{(2)}  \label{prop:properties_derived_from_definitions:m1_satDPO_m2_not_satDPO_then_m1_superior_m2} 
	If $m_1$ does and $m_2$ does not satisfy the DPO, then $m_1 \prec m_2$.\\
	\textbf{(3)}  \label{prop:properties_derived_from_definitions:vzQ'>vzQ} $Q \prec_{\DPO} Q'$ implies $\VZ_{Q'} \supset \VZ_Q$. Thus, $\VZ_{Q'} \neq \emptyset$.\\
	\textbf{(4)}  If $m$ satisfies the DPO (over $\mathfrak{X}$), then $m$ is consistent with the DPO (over $\mathfrak{X}$).\\
	\textbf{(5)}  \label{prop:properties_derived_from_definitions:construct_weaker_query_wrt_DPO} $\Pt{V}{Q'}$ of any $Q'$ satisfying $Q \prec_{\DPO} Q'$ can be obtained from $\Pt{V}{Q}$ by transferring 
	$X$ with $\emptyset \subset X \subset \VP_Q \cup \VN_Q$ to $\VZ_Q$  
	and by possibly interchanging the positions of the resulting sets $\VP_Q \setminus X$ and $\VN_Q \setminus X$, i.e.\ $\Pt{V}{Q'} = \tuple{\VP_{Q'},\VN_{Q'},\VZ_{Q'}}$ is either equal to
	$\langle\VP_Q \setminus X, \VN_Q \setminus X,\VZ_Q \cup X \rangle$
	or to 
	$\langle \VN_Q \setminus X, \VP_Q \setminus X,\VZ_Q \cup X \rangle$.
\end{proposition}
Prop.~\ref{prop:properties_derived_from_definitions}.5 substantiates the plausibility of the DPO since it shows that DPO-dispreferred queries 
result from adding some hypotheses to those ($\VZ_Q$) that cannot be invalidated by any query answer. Moreover, Prop.~\ref{prop:properties_derived_from_definitions}.3 implies that no weak query can be DPO-preferred to a strong one. Neither can a non-discriminating query be DPO-preferred to a DQ.

\textbf{Example (cont'd):}
Alternatively to directly using Def.~\ref{def:DPO} as before, Prop.~\ref{prop:properties_derived_from_definitions}.5 enables to prove 
$Q_3 \prec_{\mathsf{DPO}} Q_4$ by construction of $Q_4$ from $Q_3$ using $X := \setof{h_3}$. On the other hand, e.g.\ the DPO does not relate $Q_2$ with $Q_3$ or vice versa. This can be easily verified by Prop.~\ref{prop:properties_derived_from_definitions}.5, i.e.\ no suitable $X$ exists.  


Let $m_1,m_2$ be QSMs and their preference orders 
imposed on $V$ be (\emph{the transitive closure of}) $\{Q_1 \prec_{m_1} Q_3, Q_3 \prec_{m_1} Q_2, Q_2 \prec_{m_1} Q_4\}$ and $\{Q_1 \prec_{m_2} Q_3, Q_2 \prec_{m_2} Q_3, Q_1 \prec_{m_2} Q_4, Q_2 \prec_{m_2} Q_4\}$. Clearly, $m_1$ satisfies the DPO since its imposed order is a superset of the DPO 
$\setof{(Q_1,Q_2), (Q_3,Q_4)}$ 
over $V$ (cf.\ Def.~\ref{def:measure_satisfies_consistentWith_DPO}). $m_2$, on the contrary, is only consistent with the DPO since neither $Q_2 \prec_{m_2} Q_1$ nor $Q_4 \prec_{m_2} Q_3$ holds. It does not satisfy the DPO since e.g.\ $Q_1 \prec_{m_2} Q_2$ does not hold. So, by Prop.~\ref{prop:properties_derived_from_definitions}.2 we can conclude that $m_1$ is $\mathfrak{X}$-superior to $m_2$, i.e.\ $m_1 \prec_{\mathfrak{X}} m_2$ where $\mathfrak{X} := \{Q_1,\dots,Q_4\}$. Let $Q_4 \prec_{m_3} Q_3$ for some QSM $m_3$, then $m_3$ neither satisfies nor is consistent with the DPO.

By Prop.~\ref{prop:properties_derived_from_definitions}.3, no $Q_j$ can be discrimination-preferred to $Q_1$ or $Q_3$ since $\VZ_{Q_i} = \emptyset$ for $i \in \setof{1,3}$.
\qed

\renewcommand{\arraystretch}{1.2}
\begin{table*}[t!]
	\caption{QSMs $m$ (col.~2) grouped by query selection frameworks (QS-FWs) (col.~1). Functions $m(Q)$ (col.~3) are optimized for arguments $Q$ that maximize ($\nearrow$) or minimize ($\searrow$) $m(Q)$ (col.~4). $\checkmark$ means $m$ satisfies the DPO, $(\checkmark)$ that $m$ is consistent with, but does not satisfy the DPO, and $\times$ that $m$ is not consistent with the DPO (col.~5). Col.~6 reports whether ($\checkmark$) or not ($\times$) a theoretical optimum exists for the QSM.
		%
		%
		Numbers $_{\textbf{i)}}$ are explained below the table. Statements such as $_{(z>2)}$ state conditions under which a property holds. } 
	\label{tab:measures_satisfy_DPR_theoretical_opt_exists}
	\scriptsize
	\begin{center}
	\begin{tabular}{lccccc}
		\toprule
		QS-FW 				& QSM $m$  			& $m(Q)$ & opt. & DPO				&   	$\exists$ theor. opt.	 					\\
		\midrule
		\multirow{3}{*}{US} 	& 	$\mathsf{LC}$ 		& $p(ans(Q) = a_{Q,\max})$ & $\searrow$ &	 $\times$					&  			$\checkmark$	\\
		& 	$\mathsf{M}$ 		& $p(ans(Q)=a_{Q,1}) - p(ans(Q)=a_{Q,2})$ & $\searrow$ &	 $\times$	 				&  			$\checkmark$			\\
		& 	$\mathsf{H}$ 		& $-\sum_{a \in\setof{0,1}} p(ans(Q)=a) \log_2 p(ans(Q)=a)$ & $\nearrow$ &	 $\times$	 				&  				$\checkmark$		\\
		& 	$\mathsf{GI}$	& $1- p(ans(Q)=1)^2 - p(ans(Q)=0)^2$ & $\nearrow$ &$\times$	&	$\checkmark$		\\ 
		\midrule
		
		\multirow{2}{*}{IG}		& 	$\mathsf{ENT}$ 		& $p(\VZ_Q)+\sum_{a\in\setof{0,1}} p(ans(Q)=a) \log_2 p(ans(Q)=a)$     & $\searrow$ &$\times$&	$\checkmark$		\\
		& 	$\mathsf{ENT}_z$	& $z\,p(\VZ_Q)+\sum_{a\in\setof{0,1}} p(ans(Q)=a) \log_2 p(ans(Q)=a)$ & $\searrow$ &$\times$/$\checkmark_{\textbf{4)}}$		&	$\checkmark$		\\
		\midrule
		\multirow{4}{*}{QBC}	& 	$\mathsf{SPL}$ 		& $\left|\, |\VP_Q| - |\VN_Q| \,\right| + |\VZ_Q|$  & $\searrow$ &	 $(\checkmark)$			&	 $\checkmark$		\\
		& 	$\mathsf{SPL}_z$	& $\left|\, |\VP_Q| - |\VN_Q| \,\right| + z\,|\VZ_Q|$ & $\searrow$ &	 
		$\times_{(z<1)}$/$(\checkmark)_{(z=1)}$/$\checkmark_{(z>1)}$
		&	  	$\checkmark$			\\
		& 	$\mathsf{VE}$ 		& $- \sum_{X\in\setof{\VP_Q,\VN_Q}} \frac{|X|}{|\VP_Q \cup \VN_Q|} \log_2 \frac{|X|}{|\VP_Q \cup \VN_Q|}$ & $\nearrow$ &	 $\times$					&  	$\checkmark$		\\
		& 	$\mathsf{KL}$ 		& $-\sum_{X\in\setof{\VP_Q,\VN_Q}}\frac{|X|}{|\VP_Q \cup \VN_Q|} \log_2 \frac{p(X)}{p(\VP_Q \cup \VN_Q)}$ & $\nearrow$  &	 $\times$					&  	$\times$	\\
		\midrule
		\multirow{6}{*}{EMC}	& 	$\mathsf{EMCa}$		& $2\,\left[p(ans(Q)=1) - [p(ans(Q)=1)]^2 \right] - \frac{p(\VZ_Q)}{2}$ & $\nearrow$ &	 $\times$	&  	$\checkmark$		\\
		& 	$\mathsf{EMCa}_z$	& $2\,\left[p(ans(Q)=1) - [p(ans(Q)=1)]^2 \right] - z\, \frac{p(\VZ_Q)}{2} $ & $\nearrow$ &	
		$\times_{(z< 2)}$/$\checkmark_{(z\geq 2)}$ 
		&  	$\checkmark$			\\
		& 	$\mathsf{EMCb}$		& $p(ans(Q)=1) |\VN_Q| + p(ans(Q)=0) |\VP_Q|$ & $\nearrow$ &	 $\times$					&  	$\times$		\\
		& 	$\mathsf{MPS}$ 		& $0$ if $Q$ not a strong DQ or $\left| |\VP_Q|-|\VN_Q|\right| \neq 2$, $V_{Q,\min}$ else \;$_{\textbf{1)}}$  & $\nearrow$ &	 $(\checkmark)$				&  	$\checkmark$		\\
		& 	$\mathsf{MPS}'$ 	& $-|\VZ_Q|$ if $Q$ not a strong DQ or $\left| |\VP_Q|-|\VN_Q| \right| \neq 2$, $V_{Q,\min}$ else \;$_{\textbf{1)}}$ & $\nearrow$ &	 $\checkmark$				&  	$\checkmark$		\\
		& 	$\mathsf{BME}$ 		& $|V_{Q,p,\min}|$ \;$_{\textbf{2)}}$ & $\nearrow$ &	 $\times$					&  	$\checkmark$		\\
		\midrule
		\multirow{2}{*}{RL} 	& 	$\mathsf{RIO}'$ 		& $\frac{\mathsf{ENT}(Q)}{2}+V_{Q,n}$ \;$_{\textbf{3)}}$ & $\searrow$ &	 $\times$			 		&  	$\checkmark$		\\
		& 	$\mathsf{RIO}'_z$	& $\frac{\mathsf{ENT}_z(Q)}{2}+V_{Q,n}$ \;$_{\textbf{3)}}$ & $\searrow$ &	 
		$\times$				
		&  	$\checkmark$		\\
		\bottomrule
	\end{tabular}
	\end{center}
	\renewcommand{\arraystretch}{1}
	
	
	\begin{tabular}{p{16cm}}
		\textbf{Key:} 
		\quad 
		\textbf{1):} $V_{Q,\min} := \argmin_{X\in\setof{\VP_Q,\VN_Q}}(|X|)$. \quad
		\textbf{2):}~$V_{Q,p,\min}$ is equal to $\VN_Q$ if $p(\VN_Q) < p(\VP_Q)$, to $\VP_Q$ if $p(\VP_Q) < p(\VN_Q)$, and to $0$ else. \quad
		\textbf{3):} $V_{Q,n}$ is equal to $\min\{|\VP_Q|,|\VN_Q|\} - n$ if $\min\{|\VP_Q|,|\VN_Q|\} \geq n$, and equal to $|V|$ else. $n$ denotes the minimal number of hypotheses the next query must eliminate (in the worst case)
		\cite{Rodler2013}. \quad
		\textbf{4):}~In general, $\checkmark$ holds only if $z$ is specified as per Prop.~\ref{prop:z_parametrization_and_DPO_compliance}. \\ 
		%
		%
		%
	\end{tabular}
\end{table*}

\noindent\textbf{The Discussed QSMs.} 
Next, we briefly sketch the QSMs we address in this work (see Tab.~\ref{tab:measures_satisfy_DPR_theoretical_opt_exists}), grouped by \emph{Query Selection Framework (QS-FW)} \cite{settles09}.

\noindent\emph{Uncertainty Sampling (US):} Here, the principle is to select the query about whose answer the learner is most uncertain (as per the probability measure $p$)
given the current evidence $V$. 
\emph{Least Confidence} ($\mathsf{LC}$) selects the query
whose most likely answer $a_{Q,\max}$ has least probability. 
\emph{Margin Sampling} ($\mathsf{M}$) targets the query for which the probabilities between most and second most likely label $a_{Q,1}$ and $a_{Q,2}$ are most similar. 
\emph{Entropy} ($\mathsf{H}$) prefers the query whose outcome is most uncertain wrt.\ information entropy.
\emph{Gini Index} ($\mathsf{GI}$) is actually not an AL QSM, but is borrowed from decision tree learning theory \cite{breiman1984classification}.

\noindent\emph{Information Gain (IG):} The query favored by $\mathsf{ENT}$ maximizes the information gain \cite{moret1982decision,quinlan1986induction,dekleer1987}, or equivalently, minimizes the expected a-posteriori entropy wrt.\ 
hypotheses in 
$V$. 
%
%
As proven in \cite{dekleer1987}, $\mathsf{ENT}$ can be equivalently represented as 
shown in Tab.~\ref{tab:measures_satisfy_DPR_theoretical_opt_exists}.

\noindent\emph{Query by Committee (QBC):} QBC criteria use the competing hypotheses in 
$V$
as a committee $C$. Each \emph{predicting} committee member $h \in V$ has a vote on the classification of a $Q \in \UD$, i.e.\ the committee (for $Q$) is $C = V \setminus \VZ_Q = \VP_Q \cup \VN_Q$. The query $Q$ yielding the highest disagreement among all committee members is considered most informative. There are different ways of estimating the disagreement:
\emph{Vote Entropy} ($\mathsf{VE}$) 
selects the query for which the entropy of the relative prediction frequencies is maximal.
At this, 
$|X|/|\VP_Q \cup \VN_Q|$ with $X = \VP_Q$ ($X = \VN_Q$)
is the relative prediction frequency of label $1$ ($0$).
The \emph{Kullback-Leibler-Divergence} ($\mathsf{KL}$) proposes the query 
that manifests the largest average disagreement 
between the label distributions of any 
$h \in C$ and the consensus of the entire 
$C$ (cf.\ \cite[p.~17]{settles09} for a formal specification).
By simple mathematics, one can derive that the $\mathsf{KL}$ measure has the shape as given in Tab.~\ref{tab:measures_satisfy_DPR_theoretical_opt_exists} \cite[Prop.~26]{zattach}. 
\emph{Split-In-Half} ($\mathsf{SPL}$) \cite{moret1982decision,mitchell1997machine,Shchekotykhin2012} tries to eliminate exactly half of the currently known hypotheses, i.e.\ suggests queries which split $V$ into $\VP_Q$ and $\VN_Q$, both of size $|V|/2$ (implying $|\VZ_Q|=0$). 

\noindent\emph{Expected Model Change (EMC):} The principle is to favor the query that would impart the greatest change to the current model if its label was known. Interpreted in the sense of version spaces \cite{mitchell1977version}, we view all the available evidence $V$
as ``model''. ``Maximum expected model change'' can be interpreted in a way that (a)~the expected \emph{probability mass} of invalidated hypotheses in $V$ is maximized or (b)~the expected \emph{number} of invalidated hypotheses in $V$ is maximized.
The 
resulting QSMs, which we call $\mathsf{EMCa}$ for (a) and $\mathsf{EMCb}$ for (b), are depicted in Tab.~\ref{tab:measures_satisfy_DPR_theoretical_opt_exists}. 
Further, we propose the new QSM \emph{Most Probable Singleton} ($\mathsf{MPS}$). It favors DQs with empty $\VZ_Q$ where one of $\VP_Q,\VN_Q$ is a singleton and this singleton has maximum probability. Since in this case the probability of this singleton is equal to the probability of one answer of $Q$ (cf.\ Sec.~\ref{sec:basics}), it attempts to maximize the probability of deleting the 
maximum possible number 
of hypotheses in $V$. 
The variant $\mathsf{MPS}'$ of $\mathsf{MPS}$ additionally penalizes queries $Q$ with $\VZ_Q \neq \emptyset$.  
Another new QSM we introduce is \emph{Biased Maximal Elimination} ($\mathsf{BME}$).
The idea is to achieve a bias (probability $> 0.5$) towards an answer that rules out a maximal possible number of hypotheses. 

\noindent\emph{Reinforcement Learning (RL):} A ``risk-optimization'' reinforcement learning QSM ($\mathsf{RIO}$) was introduced in \cite{Rodler2013} to overcome performance issues in terms of sample complexity of $\mathsf{SPL}$ and $\mathsf{ENT}$ given unreasonable a-priori 
probabilities. Based on the hypothesis elimination rate achieved by the already asked queries, $\mathsf{RIO}$ adapts a learning parameter 
which determines the minimum number of hypotheses $n$ the next query must eliminate (in the worst case). 
Tab.~\ref{tab:measures_satisfy_DPR_theoretical_opt_exists} gives a slightly modified version $\mathsf{RIO}'$ of $\mathsf{RIO}$ which can be expressed in closed form (cf.\ \cite[Rem.~8]{zattach}). 
Among those queries that approach $n$ best (i.e.\ minimize $V_{Q,n}$, see Tab.~\ref{tab:measures_satisfy_DPR_theoretical_opt_exists}), the best query wrt.\ the $\mathsf{ENT}$ QSM is selected. 
%

\noindent\textbf{Compliance of QSMs with the DPO.} We next discuss how far the QSMs in Tab.~\ref{tab:measures_satisfy_DPR_theoretical_opt_exists} agree with the DPO in terms of Def.~\ref{def:measure_satisfies_consistentWith_DPO}. Results are summarized in col.~5 of Tab.~\ref{tab:measures_satisfy_DPR_theoretical_opt_exists}.
\begin{proposition}
	The QSMs $\mathsf{LC}$, $\mathsf{M}$, $\mathsf{H}$,  $\mathsf{ENT}$, $\mathsf{VE}$, $\mathsf{KL}$, $\mathsf{EMCa}$, $\mathsf{EMCb}$, $\mathsf{BME}$ and $\mathsf{RIO}'$ are not consistent with the DPO. Further, $\mathsf{SPL}$ and $\mathsf{MPS}$ are consistent with, but do not satisfy the DPO.
\end{proposition}
\vspace{-10pt}
\begin{proof} (Sketch)
	We give counterexamples based on Tab.~\ref{tab:example_QPs}. 
	First, let the hypotheses probabilities $p := p_1$. Then $p(ans(Q_1)=1) = 0.4$ and $p(ans(Q_2)=1) = 0.5$.
	Hence, $Q_2 \prec_{m} Q_1$ for $m \in \setof{\mathsf{LC}, \mathsf{M}, \mathsf{H}}$. Due to the asymmetry of $\prec_{m}$ for each QSM $m$ (Prop.~\ref{prop:precedence_order_is_strict_order}), we have $\lnot(Q_1 \prec_{m} Q_2)$. But, $Q_1 \prec_{\mathsf{DPO}} Q_2$ (see Example above). Inconsistency of $m$ with the DPO follows from Def.~\ref{def:measure_satisfies_consistentWith_DPO}. In a similar way, we obtain $Q_2 \prec_{m} Q_1$ for $m \in \setof{\mathsf{VE},\mathsf{KL}}$ because $\mathsf{VE}(Q_1)=-\frac{2}{5}\log_2 \frac{2}{5}-\frac{3}{5}\log_2 \frac{3}{5} < -2\frac{1}{2}\log_2 \frac{1}{2} = \mathsf{VE}(Q_2)$ and $\mathsf{KL}(Q_1)=-\frac{2}{5}\log_2 (0.4)-\frac{3}{5}\log_2 (0.6) < -2\frac{1}{2}\log_2 \frac{1}{2} = \mathsf{KL}(Q_2)$. Further, assuming $p:=p_3$ we analogously find that $Q_4 \prec_{m} Q_3$ for $m \in \setof{\mathsf{ENT},\mathsf{EMCa},\mathsf{RIO}'}$ (letting $n := 1$ for $\mathsf{RIO}'$), and supposing $p:=p_2$ we realize that  $Q_4 \prec_{m} Q_3$ for $m \in \setof{\mathsf{EMCb},\mathsf{BME}}$. 
	
	For all $Q, Q' \in \UD$ where $Q \prec_{\mathsf{DPO}} Q'$ and $m \in \setof{\mathsf{SPL},\mathsf{MPS}}$ it can only hold that $Q \prec_{m} Q'$ or $m(Q) = m(Q')$. This can be shown using Prop.~\ref{prop:properties_derived_from_definitions}.5 and the QSM definitions. Thence, $Q' \prec_{m} Q$ cannot hold which is why $m$ is consistent by Def.~\ref{def:measure_satisfies_consistentWith_DPO}.  
\end{proof}
For the QSMs $\mathsf{ENT}$, $\mathsf{SPL}$, $\mathsf{EMCa}$ and $\mathsf{MPS}$ we can derive (parameterized) improved versions $\mathsf{ENT}_z$, $\mathsf{SPL}_z$, $\mathsf{EMCa}_z$ and $\mathsf{MPS}'$ that satisfy the DPO (see col.~3 of Tab.~\ref{tab:measures_satisfy_DPR_theoretical_opt_exists}). The idea with all these QSMs is to penalize the inclusion of hypotheses in $\VZ_{Q}$. Because, the more elements there are in $\VZ_{Q}$, the less the query $Q$ tends to be favored by the DPO. However, it is material to obey that this penalization must be as subtle as possible in order to \emph{preserve the query selection characteristics} of the respective QSM. For instance, consider the QSM $\mathsf{ENT}$ and two queries $Q,Q'$ with $\langle p(\VP_{Q}), p(\VN_{Q}), p(\VZ_{Q}) \rangle = \langle 0.01,0.99,0\rangle$ and $\langle p(\VP_{Q'}), p(\VN_{Q'}), p(\VZ_{Q'}) \rangle = \langle 0.49,0.49,0.02\rangle$. Obviously, since $\mathsf{ENT}$ favors queries with $50$-$50$ answer probability and low $p(\VZ_Q)$, it should clearly give $Q'$ preference over $Q$ although $\VZ_{Q'} \neq \emptyset$ and $\VZ_Q = \emptyset$. Using $\mathsf{ENT}_z$ with e.g.\ $z := 50$ would however imply $\mathsf{ENT}_{z}(Q) \approx 0.92 < 0.99 \approx \mathsf{ENT}_{z}(Q')$ which contradicts the nature of entropy query selection. In general, $m_z \not\equiv m_r$ for parameters $z \neq r$ and the difference regarding query selection grows with $|z-r|$.

We now state the relationship between the specified $z$-parameter 
and DPO compliance of the new QSMs:
\begin{proposition}\label{prop:z_parametrization_and_DPO_compliance}
	Ad $\mathsf{ENT}_z$ \cite[Cor.~3+4]{zattach}: Let 
	for all $Q\in\UD$ be $\min_{a\in\{0,1\}} p(ans(Q)=a) > t > 0$. Then, for any $z \geq \max \setof{-\frac{1}{2}(\log_2 t - \log_2 (1-t)),1}$, $\mathsf{ENT}_z$ satisfies the DPO over $\UD$. Further, $\mathsf{ENT}_s \prec \mathsf{ENT}_r$ for $0 \leq r < s$. 
	
	Ad $\mathsf{EMCa}_z$ \cite[Cor.~13]{zattach}: For all $z \geq 2$ and $r \geq 0$, $\mathsf{EMCa}_{z}$ satisfies the DPO and is superior to $\mathsf{ENT}_r$. 
	
	Ad $\mathsf{SPL}_z$ \cite[Prop.~19]{zattach}: $\mathsf{SPL}_z$ is (inconsistent with / consistent with, but not satisfying / satisfying) the DPO for all ($z < 1$ / $z=1$ / $z > 1$).
	%
\end{proposition}
So, whereas for $\mathsf{EMCa}_z$ and $\mathsf{SPL}_z$ a fixed $z$-value can guarantee DPO-satisfaction for any set of queries, for $\mathsf{ENT}_z$ the $z$-parameter depends on $t$. Given a set of DQs $\UD$ wrt.\ $V$, it holds that $t < \min_{h\in V} p(h)$ by the Def.\ of $p(ans(Q)=a)$ (cf.\ Sec.~\ref{sec:basics}). A respective choice of $z$ as per Prop.~\ref{prop:z_parametrization_and_DPO_compliance} implies that $\mathsf{ENT}_z$ preserves the DPO. 
It is moreover easy to see from the definition of $\mathsf{MPS}'$ that it satisfies the DPO.

\renewcommand{\arraystretch}{1.15}
\begin{table}
	\caption{Equivalence Classes (ECs) of QSMs wrt.\ the relations $\equiv$ and $\equiv_{\mathfrak{X}}$ (cf.\ Def.~\ref{def:measures_equivalent}). 
		$\mathfrak{X}$ is any set of queries where each $Q \in \mathfrak{X}$ satisfies $\VZ_Q = \emptyset$. 
		Circled numbers 
		$\footnotesize\circled[1]{\smaller[2]i}$
		provide reference to Tab.~\ref{tab:requirements_for_equiv_classes_of_measures_wrt_equiv_mQ}, which gives only one set of requirements for each numbered EC. 
		%
		%
	}
	\label{tab:measure_equiv_classes}
	\scriptsize
	\begin{center}
	\begin{tabular}{ll}
		\toprule
		& 		Equivalence Classes (ECs) 			\\ 
		\midrule
		\multirow{4}{8pt}{$\equiv$}					& $\setof{\mathsf{ENT}_1,\mathsf{ENT}}, \setof{\mathsf{ENT}_{z\,(z \notin\setof{0,1})}}, \setof{\mathsf{SPL}_1, \mathsf{SPL}},\setof{\mathsf{EMCb}},$ \\
		& $\setof{\mathsf{SPL}_{z\,(z \notin \setof{0,1})}}, \setof{\mathsf{RIO}'_1, \mathsf{RIO}'},\setof{\mathsf{RIO}'_{z\,(z \neq 1)}},\setof{\mathsf{KL}},    $ \\
		& $\setof{\mathsf{EMCa}_1,\mathsf{EMCa}}, \setof{\mathsf{EMCa}_{z\,(z \notin \setof{0,1})}},\setof{\mathsf{VE},\mathsf{SPL}_0},	$\\
		& $\setof{\mathsf{EMCa}_0,\mathsf{GI},\mathsf{LC},\mathsf{M},\mathsf{H}, \mathsf{ENT}_0},\setof{\mathsf{MPS}},\setof{\mathsf{MPS}'}, \setof{\mathsf{BME}}$ \\
		\midrule
		\multirow{3}{8pt}{$\equiv_{\mathfrak{X}}$}		&	 $ \footnotesize\circled[1]{\smaller[2]1}:\setof{\mathsf{EMCa},\mathsf{EMCa}_{z\,(z\in\mathbb{R})},\mathsf{GI},\mathsf{LC},\mathsf{M},\mathsf{H},\mathsf{ENT},\mathsf{ENT}_{z\,(z\in\mathbb{R})}},$\\ 
		& $\footnotesize\circled[1]{\smaller[2]2}:\setof{\mathsf{SPL},\mathsf{SPL}_{z\,(z\in\mathbb{R})},\mathsf{VE}},\footnotesize\circled[1]{\smaller[2]3}:\{\mathsf{RIO}',\mathsf{RIO}'_{z\,(z\in\mathbb{R})}\},$ \\
		& $\footnotesize\circled[1]{\smaller[2]4}:\setof{\mathsf{KL}},\footnotesize\circled[1]{\smaller[2]5}:\setof{\mathsf{EMCb}}, \footnotesize\circled[1]{\smaller[2]6}:\setof{\mathsf{MPS},\mathsf{MPS}'},\footnotesize\circled[1]{\smaller[2]7}:\setof{\mathsf{BME}}$														\\
		\bottomrule
	\end{tabular}
	\end{center}
\end{table}

\noindent\textbf{Equivalence Between QSMs.} Tab.~\ref{tab:measure_equiv_classes} summarizes equivalence classes (ECs) as per Def.~\ref{def:measures_equivalent} between QSMs over arbitrary queries (row $\equiv$) and over queries $\mathfrak{X}$ satisfying $\VZ_Q = \emptyset$ (row $\equiv_{\mathfrak{X}}$).
%
ECs wrt.\ $\equiv$ cluster QSMs that manifest the exact same query selection behavior in tasks (e.g.\ model-based diagnosis, abduction) where hypotheses might specify incomplete knowledge (cf.\ Sec.~\ref{sec:intro} and \ref{sec:basics}). If all hypotheses give complete information (as in many machine learning tasks), QSMs in an EC wrt.\ $\equiv_{\mathfrak{X}}$ behave equally. 
The pragmatics of the given ECs is the reduction of the possible QSM options for a certain task, i.e.\ it makes no sense to try to improve the performance of a learner by switching between QSMs of the same EC. Along with QSM superiority results below, the ECs provide a general guidance for proper QSM choice based on the type of application. 

The proofs of the stated QSM equivalences are either direct consequences of the QSMs' definitions (Tab.~\ref{tab:measures_satisfy_DPR_theoretical_opt_exists}, col.~3) or straightforward after simple algebraic transformations. E.g.\ $\mathsf{EMCa}_0 \equiv \mathsf{GI}$ since the latter can be equivalently transformed to the former by using $p(ans(Q)=0) = 1-p(ans(Q)=1)$. Further $\mathsf{LC}\equiv\mathsf{M} \equiv\mathsf{H}\equiv\mathsf{ENT}_0$ since there are only two possible query labels. Interestingly, the EC wrt.\ $\equiv$ comprising $\mathsf{GI}$ includes QSMs of three different query selection frameworks, US, IG and EMC (cf.\ Tab.~\ref{tab:measures_satisfy_DPR_theoretical_opt_exists}). Note that the ECs including $z$-parameterized QSMs represent infinitely many \emph{different} ECs, one for each setting of $z$, e.g.\ $\mathsf{ENT}_r \not\equiv \mathsf{ENT}_s$ for $r \neq s$ (cf.\ Prop.~\ref{prop:z_parametrization_and_DPO_compliance}).
Note that some of the ECs wrt.\ $\equiv$ conflate to constitute a single EC wrt.\ $\equiv_{\mathfrak{X}}$. In particular, those ECs merge which are equivalent except for their treatment of $\VZ_Q$. Hence, infinitely many ECs wrt.\ $\equiv$ reduce to mere $7$ ECs wrt.\ $\equiv_{\mathfrak{X}}$.

\begin{figure}[tb]
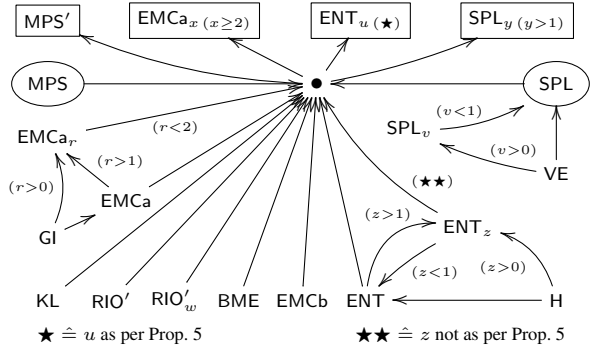

	\xygraph{
		!{<0cm,0cm>;<1.3cm,0cm>:<0cm,1.2cm>::}
		!{(0,4)}*+[F]{\scriptstyle\mathsf{MPS}'}="MPS'"
		!{(1.5,4)}*+[F]{\scriptstyle\mathsf{EMCa}_{x\,(x\geq 2)}}="EMCa_x"
		!{(3.2,4)}*+[F]{\scriptstyle\mathsf{ENT}_{u\,(\bigstar)}}="ENT_u"
		!{(4.8,4)}*+[F]{\scriptstyle\mathsf{SPL}_{y\,(y > 1)}}="SPL_y"
		!{(0,3.3)}*++[o][F]{\scriptstyle\mathsf{MPS}}="MPS"
		!{(5.2,3.3)}*++[o][F]{\scriptstyle\mathsf{SPL}}="SPL"
		!{(0,2.7)}*+{\scriptstyle\mathsf{EMCa}_r}="EMCa_r"
		!{(3.7,2.8)}*+{\scriptstyle\mathsf{SPL}_v}="SPL_v"
		!{(2.75,3.3)}*+{\bullet}="bullet"
		!{(5.2,2.3)}*+{\scriptstyle\mathsf{VE}}="VE"
		!{(0.8,2)}*+{\scriptstyle\mathsf{EMCa}}="EMCa"
		!{(3.25,0.9)}*+{\scriptstyle\mathsf{ENT}}="ENT"
		!{(4.3,1.7)}*+{\scriptstyle\mathsf{ENT}_z}="ENT_z"
		!{(0,1.6)}*+{\scriptstyle\mathsf{GI}}="Gini"
		!{(0,0.9)}*+{\scriptstyle\mathsf{KL}}="KL"
		!{(0.65,0.9)}*+{\scriptstyle\mathsf{RIO}'}="RIO"
		!{(1.3,0.9)}*+{\scriptstyle\mathsf{RIO}'_{w}}="RIO_w"
		!{(1.95,0.9)}*+{\scriptstyle\mathsf{BME}}="BME"
		!{(2.6,0.9)}*+{\scriptstyle\mathsf{EMCb}}="EMCb"
		!{(5.2,0.9)}*+{\scriptstyle\mathsf{H}}="H"
		"Gini":"EMCa"
		"EMCa_r":"bullet"_(0.45){\scriptscriptstyle(r<2)}
		"Gini":@/_0.2cm/"EMCa_r"^{\scriptscriptstyle(r>0)}
		"EMCa":"EMCa_r"_(0.4){\scriptscriptstyle(r>1)}
		"EMCa":"bullet"
		"MPS":"bullet"
		"ENT":"bullet"
		"bullet":"EMCa_x"
		"bullet":@/^0.2cm/"MPS'"
		"bullet":@/_0.1cm/"SPL_y"
		"bullet":"ENT_u"
		"VE":@/^0.1cm/"SPL_v"_(0.33){\scriptscriptstyle(v>0)}
		"VE":"SPL"
		"SPL_v":@/_0.2cm/"SPL"^(0.35){\scriptscriptstyle(v<1)}
		"H":"ENT"
		"H":@/_0.3cm/"ENT_z"^(0.4){\scriptscriptstyle(z>0)}
		"KL":"bullet"
		"RIO":"bullet"
		"RIO_w":"bullet"
		"BME":"bullet"
		"EMCb":"bullet"
		"SPL":"bullet"
		"ENT_z":@/_0.1cm/"ENT"^(0.5){\scriptscriptstyle(z<1)}
		"ENT":@/^0.5cm/"ENT_z"^(0.57){\scriptscriptstyle(z>1)}
		"ENT_z":@/^0.25cm/"bullet"_(0.3){{\scriptscriptstyle(\bigstar\bigstar)}}
	}
	\scriptsize \qquad$\bigstar \mathrel{\hat=} u$ as per Prop.~\ref{prop:z_parametrization_and_DPO_compliance} \qquad\qquad\qquad\qquad $\bigstar\bigstar \mathrel{\hat=} z $ not as per Prop.~\ref{prop:z_parametrization_and_DPO_compliance}
	\caption{QSM Superiority Relationships: $m_1 \to m_2$ denotes that $m_2 \prec m_1$ (cf.\ Def.~\ref{def:measure_superior}). Labeled arrows are conditional relations (hold only if label is true). Framed (circled) nodes indicate QSMs that satisfy (are consistent with) the DPO. Other nodes 
		are (in general) not consistent with the DPO. 
		For clarity, (1)~whenever possible, only one node for each EC in Tab.~\ref{tab:measure_equiv_classes}, row ``$\equiv$'' is depicted, and (2)~node $\bullet$ is used meaning that each incoming and outgoing arrow is to be combined.}
	\label{fig:superiority_relation_graph}
\end{figure}


\noindent\textbf{Superiority Between QSMs.} Fig.~\ref{fig:superiority_relation_graph} shows the QSM superiority relationships we derived.
Basically, these can be proven using Def.~\ref{def:measure_superior}, Prop.~\ref{prop:properties_derived_from_definitions},
the QSM functions $m(Q)$ (cf.\ Tab.~\ref{tab:measures_satisfy_DPR_theoretical_opt_exists}) and QSM equivalences (cf.\ Tab.~\ref{tab:measure_equiv_classes}). 
E.g.\ $\mathsf{ENT}_z$ for $z>0$ is superior to $\mathsf{H}$ since $\mathsf{ENT}_0 \equiv \mathsf{H}$ 
and $\mathsf{ENT}_z \prec \mathsf{ENT}_r$ for $z > r \geq 0$ by Prop.~\ref{prop:z_parametrization_and_DPO_compliance}.
%
%
Note, by Prop.~\ref{prop:properties_derived_from_definitions}.2, QSMs that satisfy the DPO (framed in Fig.~\ref{fig:superiority_relation_graph}) are proven superior to all that do not. Further, there are no $\mathfrak{X}$-superiority relationships between QSMs in the row $\equiv_{\mathfrak{X}}$ of Tab.~\ref{tab:measure_equiv_classes} by Prop.~\ref{prop:properties_derived_from_definitions}.3. In other words, the superiority graph in Fig.~\ref{fig:superiority_relation_graph} collapses over $\mathfrak{X}$ as defined in Tab.~\ref{tab:measure_equiv_classes}.

From the pragmatic viewpoint the superiority results 
are primarily relevant in a \emph{pool-based AL scenario} where a QSM is used to evaluate each query in a pool of queries and the best is selected to be shown to the oracle. Opting for a DPO-satisfying QSM then guarantees that no query is ever selected for which there is a better, i.e.\ discrimination-preferred one in the pool. However, Fig.~\ref{fig:superiority_relation_graph} must be read with care. For instance, it is not granted just due to $\mathsf{SPL}_y \prec \mathsf{KL}$ that $\mathsf{KL}$ will always manifest a worse performance (in terms of sample complexity) than $\mathsf{SPL}_y$ for $y>1$ in practice. The reason is that both QSMs follow quite different paradigms of query selection (cf.\ Tab.~\ref{tab:measures_satisfy_DPR_theoretical_opt_exists}, col.~3). Rather of interest are superiorities between \emph{related} QSMs, e.g.\ those from a particular QS-FW (cf.\ Tab~\ref{tab:measures_satisfy_DPR_theoretical_opt_exists}). For example, $\mathsf{SPL}_y$ for $y>1$ is superior to $\mathsf{SPL}$ and $\mathsf{VE}$ \emph{and implements the same preference paradigm}, attempting to eliminate half of the hypotheses in $V$. That is, (based on the parameter discussion before) one should prefer $\mathsf{SPL}_{y^*}$ (with preferably small $y^* >1$, e.g.\ $y^* := 1.1$) to the other two QSMs in pool-based AL.

\renewcommand{\arraystretch}{1.3}
\begin{table}[tb]
	\caption{Query Optimality Requirements for QSM ECs $\footnotesize\circled[1]{\smaller[2]i}$ in Tab.~\ref{tab:measure_equiv_classes}: 
		Roman numbers signalize priority, 
		i.e.\ higher numbered conditions are optimized over all queries that 
		optimize lower numbered conditions. 
	}
	\label{tab:requirements_for_equiv_classes_of_measures_wrt_equiv_mQ}
	\scriptsize
	\begin{center}
	\begin{tabular}{ll}
		\toprule
		EC 
		&   Requirements to Optimal Query  							\\
		\midrule
		${\footnotesize\circled[1]{\smaller[2]1}}$ & $\left|p(\VP_Q) - p(\VN_Q)\right| \to \min$ \\ 	
		${\footnotesize\circled[1]{\smaller[2]2}}$ &	$\left| |\VP_Q| - |\VN_Q| \right| \to \min$ \\
		${\footnotesize\circled[1]{\smaller[2]3}}$ &	$(\mathrm{I})~V_{Q,n} \to \min, (\mathrm{II})~\left|p(\VP_Q) - p(\VN_Q)\right| \to \min$ \\
		\multirow{2}{*}{${\footnotesize\circled[1]{\smaller[2]4}}$,${\footnotesize\circled[1]{\smaller[2]5}}$} &	[\,$p(\VP_Q) \to \max$ for some $|\VP_Q| \in \setof{1,\dots,|V|-1}$\,] 	$\lor$  \\ 
		&  $[\,p(\VN_Q) \to \max$ for some $|\VN_Q| \in \setof{1,\dots,|V|-1}\,]$ \\
		${\footnotesize\circled[1]{\smaller[2]6}}$ &	$(\mathrm{I})~|V^*|=1$, $V^*\in\setof{\VP_Q,\VN_Q}$, $(\mathrm{II})~p(V^*) \rightarrow \max$ \\
		${\footnotesize\circled[1]{\smaller[2]7}}$ &	$(\mathrm{I})~p(V^*)<0.5$, $V^*\in\setof{\VP_Q,\VN_Q}$, $(\mathrm{II})~|V^*| \rightarrow \max$  \\																		
		\bottomrule
	\end{tabular}
	\end{center}
\end{table}

\noindent\textbf{Properties of Optimal Queries.} We have investigated all the QSM functions $m(Q)$ in Tab.~\ref{tab:measures_satisfy_DPR_theoretical_opt_exists} wrt.\ their theoretical optima.
Most of the QSM analyses were relatively simple, e.g.\ for 
$\mathsf{SPL}$
one can easily see that no input can be better than one, say $X$, which satisfies $|\VP_X| = |\VN_X|$ and $|\VZ_X| = 0$. Moreover, for e.g.\ $m \in \setof{\mathsf{H},\mathsf{GI}}$ the existence of a theoretical optimum follows from the functions' concavity.
We report that for all discussed QSMs, except for $\mathsf{KL}$ and $\mathsf{EMCb}$, a (unique) theoretical optimum exists (Tab.~\ref{tab:measures_satisfy_DPR_theoretical_opt_exists}, last col.). 
%
In fact, analysis of the $\mathsf{KL}$ and $\mathsf{EMCb}$ functions 
yields only one stationary point which is a saddle point \cite[Prop.~27, 31]{zattach}. 

As a byproduct of studying the QSMs $m$,
we derived sufficient and necessary criteria an optimal query wrt.\ $m$ \emph{and $V$} 
must meet. Tab.~\ref{tab:requirements_for_equiv_classes_of_measures_wrt_equiv_mQ} summarizes the
results. 
Note, for $\mathsf{KL}$ and $\mathsf{EMCb}$ only necessary criteria can be named (see indeterminate 
conditions in row $\footnotesize\circled[1]{\smaller[2]4}$,$\footnotesize\circled[1]{\smaller[2]5}$). 
These can help to reduce the search space, i.e.\ optimal queries must be among those satisfying the conditions. E.g., if $Q_i,Q_j$ with $|\VP_{Q_i}|=|\VP_{Q_j}|$ and $p(\VP_{Q_i}) > p(\VP_{Q_j})$, then $Q_j$ cannot be optimal. 

The criteria in Tab.~\ref{tab:requirements_for_equiv_classes_of_measures_wrt_equiv_mQ} suggest how an optimal query wrt.\ a QSM might be \emph{systematically} constructed in a \emph{query synthesis AL scenario}. In the latter, one will usually (assuming a large enough set of unlabeled queries $\UD$) attempt to synthesize only strong DQs (cf.\ Sec.~\ref{sec:basics}), i.e.\ ones with empty $\VZ_Q$. For this reason Tab.~\ref{tab:requirements_for_equiv_classes_of_measures_wrt_equiv_mQ} just lists conditions for ECs in the $\equiv_{\mathfrak{X}}$-row of Tab.~\ref{tab:measure_equiv_classes}. In fact, the criteria target only properties \emph{of the partition} of a query (cf.\ beginning of Sec.~\ref{sec:contrib}). So, the idea is to devise \textbf{(1)} a search that enumerates DPs (cf.\ Sec.~\ref{sec:basics}) of $V$ 
in a best-first order driven by some heuristics $g_m$ derived from $m$'s optimality criteria. Once a (nearly) optimal DP is found, \textbf{(2)} a DQ for exactly this DP is generated.
%
Notably, a pro of \textbf{(1)} and \textbf{(2)} is that ideally only a single query is \emph{actually generated}. For the latter process 
might be computationally hard, e.g.\ in tasks involving logical deductions mentioned in Sec.~\ref{sec:intro}.

\begin{figure}[t]
	\begin{minipage}{0.6\textwidth} 
		\renewcommand{\arraystretch}{1}
		\setlength{\fboxsep}{0pt}
		\scriptsize 
		\xygraph{
			!{<0cm,0cm>;<1.5cm,0cm>:<0cm,1.5cm>::}
			!{(0,4.9)}*+[F]{\begin{minipage}{2.1cm}\centering
					$\scriptstyle \Part_0:$ \\
					$\scriptstyle\langle\emptyset\mid h_1,h_2,h_3,h_4 \mid\emptyset\rangle$ \\
					$\scriptstyle \mathit{probs}:\langle 0 \mid 1 \mid 0\rangle$
			\end{minipage} }="init"
			!{(0,3.9)}*+[F**:lightgray]{ {\begin{minipage}{2.1cm}\centering
						$\scriptstyle \Part_1:$ \\
						$\scriptstyle\langle h_2 \mid h_1,h_3,h_4 \mid \emptyset\rangle$ \\
						$\scriptstyle\mathit{probs}:\langle 0.15 \mid 0.85 \mid 0\rangle$ \\
						$\scriptstyle g_{\mathsf{RIO}'} = 0.067$
			\end{minipage}} }="lev1_h2"
			!{(0,2.8)}*+[F**:black][white]{ {\begin{minipage}{2.1cm}\centering
						$\scriptstyle \Part_2:$ \\
						$\scriptstyle\langle h_2,h_4 \mid h_1,h_3 \mid \emptyset\rangle$ \\
						$\scriptstyle\mathit{probs}:\langle 0.52 \mid 0.48 \mid 0\rangle$\\
						$\scriptstyle g_{\mathsf{RIO}'} = 0.02$
			\end{minipage}} }="lev2_h4"
			"init":@/^0em/^(0.45){0.15}_(0.45){h_2}"lev1_h2"
			"lev1_h2":@/^0em/^{0.37}_{h_4}"lev2_h4"
		}
		\label{fig:ex:RIO_n=2}
	\end{minipage}
	\hfill
	\begin{minipage}{0.35\textwidth}
		\centering
		\includegraphics[width=0.95\linewidth]{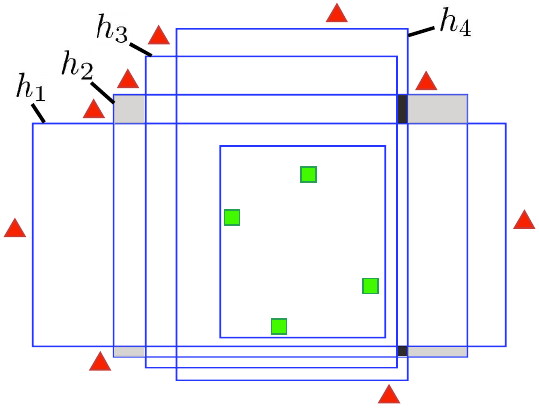} 
		\label{fig:partition_search_example}
	\end{minipage}
	\caption{Axis-parallel box classifier example \cite{settles09} (right). Heuristic search for optimal $\mathsf{RIO}'$-partition (left). Arrows point to the best successor partition as per the heuristic function $g_{\mathsf{RIO}'}$ and are labeled by the hypothesis $h_i$ and by the probability mass transferred from $\VN_Q$ to $\VP_Q$. $\mathit{probs}$ refers to $\langle p(\VP_Q) \mid p(\VN_Q)\mid p(\VZ_Q)\rangle$.}
\end{figure}

We illustrate how a (depth-first backtracking) search in \textbf{(1)} might work by means of the 
concept learning task in Fig.~\ref{fig:partition_search_example}(right) where $V = \setof{h_1,\dots,h_4}$ (rectangles), $\tuple{p(h_1),\dots,p(h_4)} = \tuple{0.41,0.15,0.07,0.37}$ and the QSM $\mathsf{RIO}'$ with $n = 2$ (cf.\ Tab.~\ref{tab:measures_satisfy_DPR_theoretical_opt_exists}) is used. 
The full version space $\V$ includes all rectangles covering all $1$-instances ($\scriptstyle\textcolor{green}{\blacksquare}$) and no $0$-instances ($\textcolor{red}{\blacktriangle}$).
Let the \emph{start partition} $\Part_0 = \langle\VP_Q,\VN_Q,\VZ_Q\rangle = \tuple{\emptyset,V,\emptyset}$, the \emph{successor function} map a partition to all neighbors resulting from the transfer of some $h \in \VN_Q$ to $\VP_Q$, the \emph{goal test} 
be $1$ iff $V_{Q,n} = 0 \land |p(\VP_Q) - p(\VN_Q)| \leq 0.05$ (cf.\ $\footnotesize\circled[1]{\smaller[2]3}$ in Tab.~\ref{tab:requirements_for_equiv_classes_of_measures_wrt_equiv_mQ}), and the \emph{heuristic function} $g_{\mathsf{RIO}'} := |p(\VP_Q)+(n-|\VP_Q|)(p(\VN_Q)/|\VN_Q|)-0.5|$. The latter returns the deviance of $p(\VP_Q)$ from $0.5$ if $|\VP_Q|=n$ is achieved by adding $n-|\VP_Q|$ further hypotheses, each with the expected probability $p(\VN_Q)/|\VN_Q|$. $g_{\mathsf{RIO}'}$ 
is used to evaluate all successors and suggests the best next one (with minimal $g_{\mathsf{RIO}'}$-value) to visit.
%
Fig.~\ref{fig:partition_search_example}(left) shows the resulting search tree (depicting only the best successors) with 3 generated partitions $\Part_0,\Part_1,\Part_2$. Note, all instances in gray and black areas, respectively, in Fig.~\ref{fig:partition_search_example}(right) are queries wrt.\ the nodes $\Part_1$ (gray) and $\Part_2$ (black) in the search tree. E.g., 
each instance $Q$ for which $\Pt{V}{Q} = \Part_1$ must be inside $h_2$ and outside of $h_1,h_3,h_4$ (cf.\ Sec.~\ref{sec:basics}).
We see that $g_{\mathsf{RIO}'}$ guides the search directly to a goal $\Part_2$. 
Generally, the algorithm could incorporate pruning criteria (devised from Tab.~\ref{tab:requirements_for_equiv_classes_of_measures_wrt_equiv_mQ}) and would backtrack if not successful along a branch. In this example a pruning rule could be to backtrack as soon as $|\VP_Q|>n$ as in this case each partition along any downward branch cannot be better than some already known one (cf.\ $V_{Q,n}$ in Tab.~\ref{tab:measures_satisfy_DPR_theoretical_opt_exists} and note that $\mathsf{ENT}(Q)/2 < 1$ by simple algebra).
Query ``generation'' in \textbf{(2)} would here be just the selection of any instance from the black areas. All of them discriminate wrt.\ $V$ as prescribed by the goal DP $\Part_2$. Note, in case no query exists for a found goal DP (as is the case e.g.\ for $\tuple{\setof{h_1,h_3},\setof{h_2,h_4},\emptyset}$), the steps \textbf{(1)} (search continuation) and \textbf{(2)} are reiterated. \cite[Sec.~3.4-3.7]{zattach} shows how \textbf{(1)} and \textbf{(2)} might be realized in the domain of model-based diagnosis, where more sophisticated successor computation in \textbf{(1)} and query generation in \textbf{(2)} must be addressed.
%

\section{CONCLUSIONS}
For active learning interpreted as search of the version space, useful for both classical machine learning and alternative problems like model-based diagnosis or abduction, we have formalized and derived
relationships between queries based on their discrimination power and between query selection measures based on their output quality. We have deduced optimality criteria for measures and introduced new (improved) variants to resolve identified issues. The results give guidance for using the right measure in pool-based active learning and suggest efficient search procedures for optimal query synthesis. 

\section{ACKNOWLEDGEMENTS}
This work was supported by the Carinthian Science Fund (KWF) contract KWF-3520/26767/38701.

\bibliographystyle{acm}
\bibliography{library}

%
%
%
%
%

\end{document}